\documentclass{article}

\usepackage{times}

\usepackage{graphicx} 
\usepackage{subfig}
\usepackage{courier}

\usepackage{tabu}
\usepackage{tabularx}
\usepackage[toc,page]{appendix}
\usepackage{xcolor, colortbl}

\usepackage[numbers]{natbib}

\usepackage{algorithm}
\usepackage{algorithmic}

\usepackage{hyperref}

\usepackage{amsmath}
\usepackage{amsthm}
\usepackage{amsfonts}

\usepackage[T1]{fontenc}
\usepackage{lmodern}

\usepackage{geometry}
\usepackage{subfig}
\usepackage[explicit]{titlesec}

\newtheorem{thm}{Theorem}
\newtheorem{lemma}{Lemma}
\newtheorem{prop}{Proposition}

\newcommand{\E}{\mathrm{E}}

\newcommand\tab[1][0.5cm]{\hspace*{#1}}

\title{Linear Regression with Shuffled Labels}
\author{Abubakar Abid$^1$ \tab Ada Poon$^1$ \tab James Zou$^{1,2}$}

\date{%
    $^1$Department of Electrical Engineering\\%
    $^2$Department of Biomedical Data Science\\
    Stanford University
    \\[2ex]%
    \today
}

\begin{document} 

\makeatletter
\hfil\parbox[t]{0.8\textwidth}{\centering\LARGE\bfseries\@title}\par
\kern0.7cm \hrule\kern0.5cm
\hfil\parbox[t]{0.7\textwidth}{\centering\large\@author\\[3ex] \normalsize \@date}\par
\kern0.5cm \hrule\kern0.5cm
\makeatother






\begin{abstract} 
Is it possible to perform linear regression on datasets whose labels are shuffled with respect to the inputs? We explore this question by proposing several estimators that recover the weights of a noisy linear model from labels that are shuffled by an unknown permutation. We show that the analog of the classical least-squares estimator produces inconsistent estimates in this setting, and introduce an estimator based on the self-moments of the input features and labels. We study the regimes in which each estimator excels, and generalize the estimators to the setting where partial ordering information is available in the form of experiments replicated independently. The result is a framework that enables robust inference, as we demonstrate by experiments on both synthetic and standard datasets, where we are able to recover approximate weights using only shuffled labels. Our work demonstrates that linear regression in the absence of complete ordering information is possible and can be of practical interest, particularly in experiments that characterize populations of particles, such as flow cytometry.

\end{abstract} 

\section{Introduction}
\label{section:intro}
Since at least the 19th century, linear regression has been widely used in statistics to infer the relationship between one more explanatory variables (or \textit{input features}) and a continuous dependent variable (or \textit{label}) \cite{stanton2001galton, seber2012linear}. In the classical setting, linear regression is used on supervised datasets that are fully and individually labeled. Not all data fit this criterion, so, in recent years, the question of inference from \textit{weakly-supervised} datasets has drawn attention in the machine learning community \cite{crandall2006weakly, li2013convex, zantedeschi2016beta}. In weakly-supervised datasets, data are neither entirely labeled nor entirely unlabeled; a subset of the data may be labeled, as is the case in semi-supervised learning, or the data may be implicitly labeled, as occurs, for example, in multi-instance learning \cite{zhou2011semi, zhou2007relation}. Weakly-supervised datasets naturally arise in situations where obtaining labels for individual data is expensive or difficult; often times, it is significantly easier to conduct experiments that provide partial information. 

In this paper, we study one specific case of weakly-supervised data: shuffled data, in which all of the labels are observed, but the mutual ordering between the input features and the labels is unknown. Shuffled linear regression, then, can be described as a variant of traditional linear regression in which the labels are additionally perturbed by an unknown permutation. More concretely, the learning setting is defined as follows: we observe (or choose) a matrix of input features $\boldsymbol{x} \in \mathbb{R}^{n \times d}$, and observe a vector of output labels $y \in \mathbb{R}^{n}$ that is generated by the following process: 
\begin{equation}
\label{eqn_model}
y = \boldsymbol{\pi_0} \boldsymbol{x} w_0 + e
\end{equation}
where $\boldsymbol{\pi_0}$ is an unknown $n \times n$ permutation matrix, $w_0 \in \mathbb{R}^{d}$ are unknown weights, and $e \in \mathbb{R}^{n}$ is additive Gaussian noise. Here, $n$ is the number of data points, and $d$ is the dimensionality. The model described by (\ref{eqn_model}) and illustrated in Figure \ref{fig:model} frequently occurs in experiments that simultaneously characterize or analyze a large number of objects. Consider an example.

\begin{figure}[]
\centering
\includegraphics[width=0.7\columnwidth]{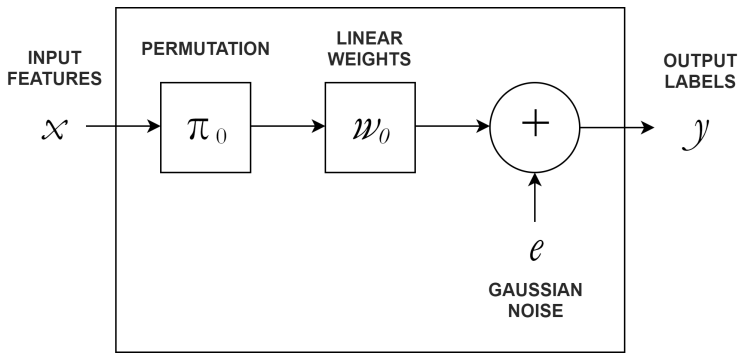}
\caption{We study shuffled linear regression, the problem of estimating linear weights $w_0$ given observed input features $\boldsymbol{x}$ and output labels $y$ in the presence of an unknown permutation $\boldsymbol{\pi_0}$ and additive Gaussian noise, $e$.}
\label{fig:model}
\end{figure}

\paragraph*{Example: Flow Cytometry.} The characterization of cells using flow cytometry usually proceeds by suspending cells in a fluid and flowing them through a laser that excites components within or outside the cell. By measuring the scattering of the light, various properties of the cells can be quantified, such as granularity or affinity to a particular target \cite{shapiro2005practical}. These properties (or labels) may be explained by features of the cell, such as its gene expression, a relationship that we are interested in modeling. However, because the \textit{order} of the cells as they pass through the laser is unknown, traditional inference techniques cannot be used to associate these labels with features of the cells that are measured through a separate experiment.


\paragraph*{} Shuffled regression is also useful in contexts where the order of measurements is unknown; for example, it arises in signaling with identical tokens \cite{rose2014signaling}, simultaneous pose-and-correspondence estimation in computer vision \cite{david2004softposit}, and in relative dating from archaeological samples \cite{robinson1951method}. A particularly important setting where the feasibility of shuffled regression raises concern is data de-anonymization, such as of public medical records, which are sometimes shuffled to preserve privacy \cite{li2004protection}.

In this paper, we show that it is possible to perform inference in these settings. We propose an intuitive algorithm for shuffled linear regression that is able to estimate the weights $w_0$ with reasonable accuracy without knowledge of the ordering of the labels $y$. The algorithm can can be applied with minor modification to settings where partial ordering information is available in the form of multiple independent experiments that each generate a set of shuffled labels for a set of input features. In fact, we show that these additional experiments or \textit{replications} significantly reduce the inference error. Our approach to shuffled linear regression takes the form outlined in Algorithm \ref{alg:example}, which numerically minimizes an objective function that is, in general, non-convex.

\begin{algorithm}[]
   \caption{Shuffled Linear Regression Using Multi-Start Gradient Descent}
   \label{alg:example}
\begin{algorithmic}
   \STATE {\bfseries Inputs:} features $\boldsymbol{x}_1 \ldots \boldsymbol{x}_R$ and labels $y_1 \dots y_R$ from $R$ replications; loss function $L(\boldsymbol{x},y,w)$; number of starts $s$; step size $\alpha$
   \STATE Initialize $\hat{w}^* = 0$, set $C^* = \infty$
   \FOR{$i=1$ {\bfseries to} $s$}
   \STATE Randomly initialize $\hat{w}$
   \REPEAT
        \STATE set C = 0
      \FOR{$j=1$ {\bfseries to} $R$}
   \STATE Evaluate the loss function, $L$, using
   $\boldsymbol{x}_j$, $y_j$, and $\hat{w}$ 
   \STATE $C \leftarrow C + L$ 
    \ENDFOR
      \STATE Empirically find the gradient of $C$ around $\hat{w}$,  $\nabla C(\hat{w})$
   \STATE Update $\hat{w} \leftarrow \hat{w} - \alpha \nabla C(\hat{w})$ 
   \UNTIL{the change in $C$ between iterations is less than a threshold (e.g. $10^{-6}$)}
   \IF{ $C(\hat{w}) < C^*$}
   \STATE Update $\hat{w}^* \leftarrow \hat{w}, C^* \leftarrow C(\hat{w})$
   \ENDIF
   \ENDFOR
   \STATE return $\hat{w}^*$
\end{algorithmic}
\end{algorithm}

We begin developing this algorithm in Section \ref{section:estimators} by proposing several estimators that minimize a loss function defined on the input features and unordered labels. We analyze the statistical properties of these estimators in low-dimensional settings, and show empirical results in higher dimensions. In Section \ref{section:framework}, we consider three steps that can be taken to improve the accuracy of estimates in real-world applications -- most significantly, replications. Next, in Section \ref{section:experiments}, we demonstrate our algorithm on synthetic and standard datasets, as well as on a real-world application: aptamer evolution. Finally, we discuss the significance of our results (Section \ref{section:discussion}) and offer a conclusion of our work (Section \ref{section:conclusion}). Our code is fully available on GitHub\footnote{\url{
https://github.com/abidlabs/shuffled_stats}} and on the Python Package Index \footnote{Available with: \texttt{pip install shuffled\char`_stats}}.

\paragraph*{Related Work.} There are two inference problems inherent to shuffled regression: estimating the permutation matrix, $\boldsymbol{\pi_0}$, and estimating the weights, $w_0$. Recently, there have been studies on the statistical and computational limits of permutation recovery in shuffled regression \cite{pananjady2016linear} and the theoretical guarantees of the identifiability of weights in the absence of noise \cite{unnikrishnan2015unlabeled}. Neither of these works directly translates to our setting, as we focus on inference of the weights, not the permutation matrix, which have very different statistical limits, and we are concerned with efficient algorithms for estimation in noisy settings. 

There has also been prior work on determining classifiers using proportional labels \cite{stolpe2011learning, yu2013propto}. In this setting, the authors aim to predict discrete classes that label instances belong to from training data that is provided in groups and only the proportion of each class label is known. This can be viewed as the corresponding problem of ``shuffled classification'' and efficient algorithms have been proposed for the task. However, these methods cannot be easily applied when labels are continuous numbers, motivating the need for algorithms for shuffled regression.

\section{Estimators for Shuffled Regression}
\label{section:estimators}

We consider estimators that can be used to perform shuffled linear regression. We prove basic properties of these estimators in lower dimensions, and show empirical performance in higher dimensions. 

The random design setting that we use for this section is one in which the entries of the $i^{\mathrm{th}}$ column of the input matrix $\boldsymbol{x}$ are drawn identically and independently from a Gaussian distribution with mean $\mu_{X_i}$ and variance $\sigma_{X_i}^2$. In the one-dimensional case, or when all columns are drawn from the same distribution, we define $\mu_{X} \equiv \mu_{X_1}$ and $\sigma_{X} \equiv \sigma_{X_1}.$ The noise source is assumed to be a zero-mean Gaussian variable with variance $\sigma_E^2$, and the permutation matrix $\pi_0$ to be chosen at random uniformly from the set of all $n \times n$ permutation matrices.

Throughout this paper, non-bold lower-case letters ($y$, $e$, $n$, $\hat{w}$, $\mu_X$) are used to denote vectors or scalars, bold lower-case letters are used to denote matrices ($\boldsymbol{x}$, $\boldsymbol{\pi_0}$), and upper-case letters are used to denote random variables ($X_2$, $Y$, $E$), unless defined otherwise.

\subsection{LS Estimator}

In standard linear regression, the ordinary least-squares (OLS) estimator is defined as

\begin{equation}
\label{eqn_LS_standard}
\hat{w}_{\text{OLS}} = \arg \min_w \, \lvert \boldsymbol{x} w - y \rvert^2 .
\end{equation}

The natural extension of (\ref{eqn_LS_standard}) in the case of shuffled linear regression is to search for the weights that minimize the least-square distance across all valid $n \times n$ permutation matrices, $\boldsymbol{\pi}$:

\begin{equation}
\label{eqn_LS_shuffled}
\hat{w}_{\text{LS}} = \arg \min_w  \min_{\boldsymbol{\pi}} \, \lvert \boldsymbol{\pi} \boldsymbol{x} w - y \rvert^2 .
\end{equation}

We define this to be the least-squares (LS) estimator for shuffled regression, and it has been studied in prior work for the recovery of the permutation matrix \cite{pananjady2016linear}. Although the OLS estimator is well known to be consistent for standard linear regression, the LS estimator is not consistent for shuffled linear regression, as we prove in Theorem \ref{thm_LS_inconsistency} for $d=1$ (see appendix \ref{appendix:proof1} for proof).

\begin{thm}
\label{thm_LS_inconsistency}
Let $\boldsymbol{x} \in  \mathbb{R}^{n \times 1}$ be sampled from a Gaussian random variable with mean $\mu_X$ and variance $\sigma_X^2$. Let $y = \boldsymbol{\pi_0 x} w_0 + e$ be the product of $\boldsymbol{x}$ with an unknown scalar weight $w_0$ and an unknown $n \times n$ permutation matrix $\boldsymbol{\pi_0}$ added to zero-mean Gaussian noise with variance $\sigma_E^2$. The least-squares estimate of $w_0$, as defined in (\ref{eqn_LS_shuffled}), is inconsistent. In fact, the estimator converges to the following limit with probability 1:
\begin{equation}
\lim_{n \to \infty} \hat{w}_{\mathrm{LS}}  = w_0 \left( \frac{\mu_X^2 + \sigma_X \sqrt{\sigma_X^2 + \frac{\sigma_{E}^2}{w_0^2}}}{\mu_X^2 + \sigma_X^2} \right) \end{equation}
\end{thm}

The limit defined in Theorem \ref{thm_LS_inconsistency} is greater in magnitude than $w_0$ whenever $\sigma_X, \sigma_e \ne 0$. This \textit{amplification bias} means that in the presence of a Gaussian noise source, even for arbitrarily large sample sizes, the LS estimator does not converge to the true weight, $w_0$. We have not extended the proof to higher dimensions, but empirical results suggest that the LS estimator is generally inconsistent and the amplification bias persists in that the norm of estimated weights is larger than that of the true weights even for $d>1$.

It may appear at first that using the LS estimator to estimate the weights is computationally intractable, because it requires searching over the space of $n!$ permutation matrices. However, the LS estimator is actually computationally feasible. This is because we can eliminate the optimization over permutations by rewriting (\ref{eqn_LS_shuffled}) as

\begin{equation}
\label{eqn_LS_simplified}
\begin{split}
&L(\boldsymbol{x},y,w) =  \lvert  (\boldsymbol{x} w)^{\uparrow} - y^{\uparrow} \rvert^2
\\
&\hat{w}_{\text{LS}} = \arg \min_w L(\boldsymbol{x},y,w)
\end{split}
\end{equation}

Here, we use the notation $v^{\uparrow}$ to refer to the vector that results from sorting the entries of a vector $v$ in ascending order. (The justification that can rewrite the LS estimator in this form follows from the proof of Lemma \ref{lemma:simplerls}, see appendix \ref{appendix:proof1}). Thus, evaluating the objective function in (\ref{eqn_LS_simplified}) for a given $w$ requires only performing a dot product and sorting the resulting array, which can be done in polynomial time. 

\subsection{SM Estimator}

The LS estimator is not consistent for shuffled regression when $d=1$. Is it possible to devise an estimator that is consistent in this setting? For inspiration, we turn to the method-of-moments (MOM) estimator \cite{hall2005generalized}. MOM generally incorporates both the \textit{self-moments}, such as the mean of $y$ or the variance of the second column of $\boldsymbol{x}$, as well as the \textit{cross-moments}, such as the co-variance of the first column of $\boldsymbol{x}$ and $y$.

In the case of shuffled linear regression, without knowing the mutual ordering between $\boldsymbol{x}$ and $y$, we can only calculate the self-moments of $\boldsymbol{x}$ and $y$. Thus, the self-moments (SM) estimator, $\hat{w}_\text{SM}$ (or simply $\hat{w}$ in this section), estimates $w_0$ by constraining moments of $\boldsymbol{x} \cdot w$ to equal the respective moments of $y$. The moments involved in these constraints depend on the value of $d$, and we discuss specific cases below.

\subsubsection{$d=1$}

For $d=1$, $\boldsymbol{x}$ is one-dimensional (there is no separate intercept term). So we require only one constraint to estimate $w_0$, which is a scalar. We can thus write down the following constraint for $\hat{w}_{\text{SM}}$, based on the sample means:
\begin{equation}
\label{eqn_sm_d1}
\frac{1}{n} \sum_{i}^n x_i \hat{w}_{\text{SM}} = \frac{1}{n} \sum_{i}^n y_i \implies \hat{w}_{\text{SM}} = \frac{\sum_{i}^n y_i}{\sum_{i}^n x_i},
\end{equation}

where $x_i$ is used to refer to the $i^{th}$ entry of the vector $\boldsymbol{x}$, and the same for $y_i$ and $y$. Theorem \ref{thm_sm_consistency} guarantees that with a mild condition on the mean of $\boldsymbol{x}$, the SM estimator is unbiased and consistent.

\begin{thm}
\label{thm_sm_consistency}
Let $\boldsymbol{x} \in  \mathbb{R}^{n \times 1}$ be sampled from a Gaussian random variable with mean $\mu_X$ and variance $\sigma_X^2$. Let $y = \boldsymbol{\pi_0 x} w_0 + e$ be the product of $\boldsymbol{x}$ with an unknown scalar weight $w_0$ and an unknown $n \times n$ permutation matrix $\boldsymbol{\pi_0}$ added to zero-mean Gaussian noise with variance $\sigma_E^2$. The SM estimator, as defined in (\ref{eqn_sm_d1}), is consistent when $\mu_X \ne 0$. In other words, with probability 1,
\begin{equation*}
\label{consistent_sm}
\lim_{n \to \infty} \hat{w}_{\mathrm{SM}}  = w_0. \end{equation*}
Furthermore, given a fixed sample vector $x$, the SM estimator is unbiased as long as as $\bar{x} \equiv \frac{1}{n}\sum_i x_i \ne 0$, meaning that:
\begin{equation*}
\label{unbiased_sm}
 \E[\hat{w}_{\mathrm{SM}} - w_o] = 0.
\end{equation*}
\end{thm}

As a result of Theorem \ref{thm_sm_consistency} (see Appendix \ref{appendix:proofthm2} for proof), when $d=1$, the SM estimator often performs better than the LS estimator.  In fact, it achieves an efficiency similar to that of the OLS estimator \textit{with known ordering information}. The empirical errors of the the LS estimator, the SM estimator, and the OLS estimator are shown for various values of $n$ in Figure \ref{fig:figure2}. The SM estimator also has computational advantages over the LS estimator, as $\hat{w}$ can be calculated analytically, requiring only a single summation over the input and label vectors, which can be done in linear time. But what if $d>1$ (or $\bar{x}=0$)? We turn to higher moments.

\begin{figure}[]
\centering

\includegraphics[width=0.6\columnwidth]{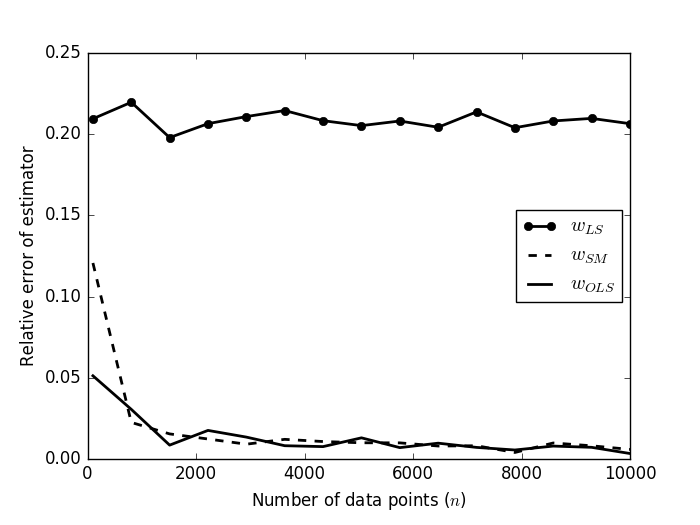}
\caption{An illustration of the (in)consistency and efficiency of the LS, SM, and OLS estimators for $d=1$. Here, $\sigma_E^2 = 1$, $w_0$ is fixed to $[\,1\,]$, and the entries in $\boldsymbol{x}$ are generated from the Gaussian distribution, $N(1,1)$. The relative error is defined as $\lvert \hat{w} - w_0 \rvert_2/\lvert w_0 \rvert_2$. The LS estimator is not consistent. The SM estimator is consistent, and converges almost as quickly (without ordering information) as the OLS estimator (with ordering information).} 
\label{fig:figure2}
\end{figure}

\subsubsection{$d=2$}
\label{section:smd2}

When $d=2$, the constraint defined in (\ref{eqn_sm_d1}) is no longer enough to uniquely identify $\hat{w}$, as there are now 2 separate variables that need to be solved for: $\hat{w}_1$ and $\hat{w}_2$. We can add another constraint by looking at the next higher self-moment of $y$ and the columns of $\boldsymbol{x}$. By considering both the first and second moments, we write:

\begin{equation}
\label{eqn_sm_d2}
\begin{split}
\frac{1}{n} \sum_{i}^n x_i \hat{w} =& \frac{1}{n} \sum_{i}^n y_i \\ \frac{1}{n} \sum_{i}^n (x_i \hat{w})^2 =& \frac{1}{n} \sum_{i}^n y_i^2, 
\end{split}
\end{equation}
where $x_i$ now refers to the $i^{th}$ row of matrix $\boldsymbol{x}$. The system of equations can also be solved analytically to provide a closed-form solution for $\hat{w}$, computable with a running time that is linear in $n$. This solution is provided in appendix \ref{appendix:analyticalsm2}. These constraints actually generally yield two solutions for $\hat{w}$, which can be disambiguated by examining even higher moments or using the LS estimator, if needed. 

However, the SM estimator is no longer consistent for $d=2$, because the second moment of $y$ now includes a contribution from the variance of the noise source. To adjust for this, we can incorporate the noise variance into our model (see Section \ref{subsection:noise}) if it is known or if it can be measured.

\subsubsection{$d > 2$}

For $d>2$, there are $d$ separate unknown linear weights, so we must write $K \ge d$ equations, each one incorporating a higher moment. The $K$ equations are, for $1 \le k \le K$, constraints of the following form:

\begin{equation}
\label{eqn_na_sm_dgt2}
\underbrace{\frac{1}{n}\sum_{i}^n (x_i w)^k}_{M_k} = \underbrace{\frac{1}{n}\sum_{i}^n y_i^k}_{N_k}. 
\end{equation}

For convenience, we denote the relevant expression for the $k^\mathrm{th}$ moment of $y$ in (\ref{eqn_na_sm_dgt2}) as $N_k$ and the expression for the  $k^\mathrm{th}$ moment of $\boldsymbol{x} \cdot \hat{w}$ as $M_k$. It is generally not possible to solve these $D$ equations analytically; in fact, there may not be a solution to the system. As a result, we instead write a single cost function that minimizes the extent that each of the $D$ constraints are violated:

\begin{equation}
\label{eqn_cost_function}
\begin{split}
&L(\boldsymbol{x},y,w) = \sum_{k=1}^K f(k)(M_k - N_k)^2\\
& \hat{w}_{\text{SM}} = \arg \min_w L(\boldsymbol{x},y,w).
\end{split}
\end{equation}

Here, the choice of the function $f(k)$ affects the relative contribution from each moment condition. One possible function is $f(k) = k!^{-1}$, which weighs each moment inversely proportional to the expected variance of the $k^{th}$ sample moment of samples taken from a Gaussian distribution \cite{lancaster1970advanced}. Once we have specified (\ref{eqn_cost_function}), we minimize it across $\hat{w}$, using standard numerical optimization techniques, as described in Section \ref{subsection:algorithm}.

As is the case with $d=2$, this estimator is not consistent. To make it consistent, we need to incorporate moments of the noise source; this is discussed further in Section \ref{subsection:noise}. Furthermore, we note that higher sample moments are likely to exhibit significant variance \cite{lancaster1970advanced}, reducing the efficiency of the SM estimator for increasing values of $d$. As a result, we find that the SM estimator can have the same or worse efficiency than the LS estimator when $d$ is even just $3$, as illustrated in Figure \ref{fig:figure2b}.

\begin{figure}[]
\centering

\includegraphics[width=0.6\columnwidth]{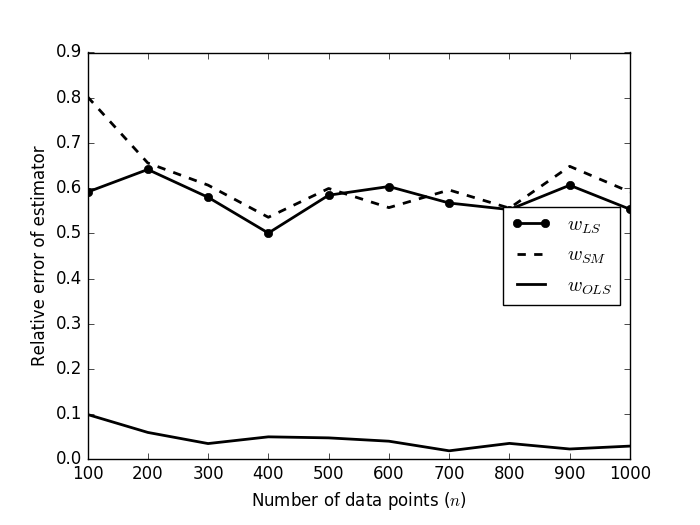}
\caption{An illustration of the (in)consistency and efficiency of the LS, SM, and OLS estimators for $d=3$. Here, $\sigma_E^2 = 1$, $w_0$ is fixed to $[\,1\,,\,1\,,1\,]$, and the entries in $\boldsymbol{x}$ are generated from the Gaussian distribution, $N(1,1)$. The relative error is defined as $\lvert \hat{w} - w_0 \rvert_2/\lvert w_0 \rvert_2$. Neither the LS estimator nor the SM estimator are consistent, and exhibit similar inference errors.}
\label{fig:figure2b}
\end{figure}

\subsection{Hybrid Estimators Using Projections}

The SM estimator has the advantage of being consistent (for $d=1$ and, with knowledge of the noise characteristics, for higher dimensions as well). However, in the higher-dimensional setting, it exhibits significant error because it requires computing higher sample moments. In this setting, the LS estimator is more efficient. We can construct a hybrid estimator that has advantages of both. 

We do this by projecting $\boldsymbol{x}$ to a lower dimension, and then using the SM estimator. For example, let us consider the specific case of projecting $\boldsymbol{x}$ to 1-dimensional space, by using a $d \times 1$ projection matrix, $p$, before using the SM estimator. The result is an estimate $\tilde{w}$ defined by

\begin{equation}
\label{projection_p1}
\frac{1}{n} \sum_{i}^n x_i p \tilde{w} = \frac{1}{n} \sum_{i}^n y_i \implies \tilde{w} = \frac{\sum_{i}^n y_i}{\sum_{i}^n x_i p},
\end{equation}

which can be embedded in the original $d$-dimensional space using the same matrix $p$ to produce $\hat{w}_{\text{P1}} \equiv p \tilde{w}$. Note that $\hat{w}_{\text{P1}}$ satisfies the first-moment condition, namely that the mean of $\boldsymbol{x} \cdot \hat{w}_{\text{P1}}$ is the same as the mean of $y$, since:

\begin{equation}
\label{projection_p1_mean} \frac{1}{n} \sum_{i}^n x_i \hat{w}_{\text{P1}} = \frac{1}{n} \sum_{i}^n x_i p \tilde{w} = \frac{1}{n} \sum_{i}^n y_i
\end{equation}

However, $\hat{w}_{\text{P1h}}$ may still be a poor estimate of $w_0$ -- it depends entirely on the choice of the projection matrix $p$. This suggests searching over the \textit{projection matrix}, instead of directly over the weights, and then using the criterion of minimizing the least-squares difference between $(\boldsymbol{x} \cdot \hat{w}_{\text{P1}})^{\uparrow}$ and $y^{\uparrow}$ to determine the best $p$.

This approach, which we refer to as the P1 estimator, effectively applies the LS estimator \textit{only} to those weights which satisfy at least the first-moment condition. We can similarly extend this approach to the second moment by projecting to 2 dimensions (denote this as the P2 estimator), and so on. The general approach is outlined in Algorithm \ref{alg:projection}.

\begin{algorithm}[]
   \caption{Hybrid Estimators Using Projections}
   \label{alg:projection}
\begin{algorithmic}
   \STATE {\bfseries Input:} $d$-dimensional features $\boldsymbol{x}$, labels {y}, projection dimension $d_p$.
   \STATE Initialize $\hat{w}^* = 0, L^* = \infty$
   \STATE Randomly initialize a $d \times d_p$ matrix $p$
   \REPEAT
   \STATE Project $\boldsymbol{x}$ to $d_p$-dimensional space: $\tilde{\boldsymbol{x}} \equiv \boldsymbol{x} \cdot p$
   \STATE Calculate $\tilde{w}$ using the SM estimator applied to $\tilde{\boldsymbol{x}}$ and $y$. 
   \STATE Let $\hat{w} \equiv p \cdot \tilde{w}$ be the higher-dimensional embedding of the weights
   \STATE Let $L = \lvert  (\boldsymbol{x} \cdot \hat{w})^{\uparrow} - y^{\uparrow} \rvert^2$ be the squared loss
    \IF{ $L < L^*$}
   \STATE Update $\hat{w}^* \leftarrow \hat{w}, L^* \leftarrow L$
   \ENDIF
   \STATE Choose a new $p$ (by gradient descent, randomly, etc.)
   \UNTIL{the change in $L^*$ is below a threshold}
   \STATE return $\hat{w}^*$
\end{algorithmic}
\end{algorithm}

\subsection{Regimes of Efficient Operation}
\label{section:regimes}

We have discussed several estimators for shuffled regression. Which one should we use? Empirical simulations suggest that the answer depends on $n$ and $d$, as well as the level of noise in the system. Simulations across a variety of noise levels suggest that the SM estimator works well for lower dimensions and larger sample sizes, while for higher dimensions, the P1 and LS estimators work better. The SM and P1 estimators are compared in Figure \ref{fig:regimes1}, and more details are provided in Appendix \ref{appendix:regimes}.

\begin{figure}[]
\centering
\includegraphics[width=0.7\columnwidth]{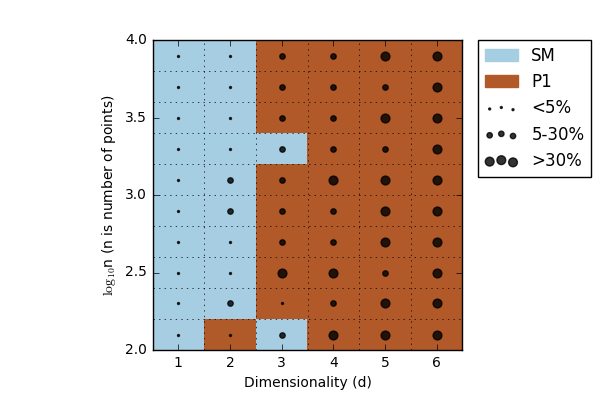}
\caption{For different values of the sample size $n$ and dimensionality $d$ the estimator with the better performance (SM or P1) is shown. Specifically, at each operating point, 5 simulations were run with an SNR fixed to 15 dB, and the estimator with the lower mean inference error is shown. The magnitude of the error is indicated by the size of the black dot at the center of each cell. The LS estimator, not shown, performs similarly to the P1 estimator (see Appendix \ref{appendix:regimes} for more details).}
\label{fig:regimes1}
\end{figure}

Finally, let us note that the LS, SM, and hybrid estimators are specific examples of a general class of estimators that may be used for shuffled regression. For reasons that we describe in Appendix \ref{appendix:order-invariant}, we refer to the general class of estimators as \textit{order-invariant estimators}. In the appendix, we provide several other examples, and explain why we chose to analyze these specific cases.

\section{A Robust Framework}
\label{section:framework}

Although the estimators described in Section \ref{section:estimators} can be used directly to perform linear regression, we can often improve the accuracy of the estimates by taking into account additional information known about the data. In this section, we describe three techniques to improve the robustness of the shuffled linear regression by taking advantage of partial ordering information, characteristics of the noise, and sparsity.

\subsection{Partial Ordering Information}
\label{subsection:partial}

In weakly-supervised settings, partial ordering information is sometimes available in the form of \textit{replications}. Each replication associates a set of input features to a set of labels, but within a replication, the mutual ordering information between the inputs and labels is unknown. This frequently arises, for example, when multiple rounds of a population-level experiment are conducted, and each experiment generates one set of inputs and outputs.

Having partial ordering information significantly improves the accuracy of shuffled linear regression, because we can write loss functions for each replication independently, and then take their sum to produce a single loss function that is more robust. For example, we can adapt the SM estimator by rewriting the loss function in (\ref{eqn_cost_function}) as: 

\begin{equation}
\label{eqn:cluster_loss}
\begin{split}
& L(\boldsymbol{x},y,w) = \sum_{r=1}^R \sum_{k=1}^K f(k)(M_{r,k} - N_{r,k})^2, \text{for} \\
&  M_{r,k} \equiv \frac{1}{|I_r|} \sum_{i \in I_r}  ( x_i w)^k
  \\
& N_{r,k} \equiv \frac{1}{|I_r|} \sum_{i \in I_r} y_i^k
\end{split}
\end{equation}
where $I_r$ represents the set of points in the $r^{\mathrm{th}}$ replication (out of a total of $R$ replications). 

The SM estimator improves with more replications because it reduces the need for higher-order sample moments in order to uniquely estimate $w_0$. In other words, increasing $R$ allows one to decrease $K$. Since it is the higher sample moments that are more variable, the estimator defined by (\ref{eqn:cluster_loss}) for a larger $R$ is generally more efficient. In a similar way, the LS estimator can be written by minimizing the sum of squared differences across separate replications, and empirical results show similarly increased accuracy with more replications. 

In Figure \ref{fig:figure3}, we show the effect of increasing the number of replications by plotting the inference error of the SM estimator as a function of the number of replications. To eliminate the confounding effect of increasing number of $n$, we have kept $n$ fixed, and only changed the numbers of replications that the data has been randomly partitioned into. Further experiments in Section \ref{subsection:standard} examine the effect of replications on inference error in real-world datasets. 

We also show the effect of partial information on the regimes of efficient operation. In Figure \ref{fig:figure2-2}, we present empirical results for the optimal estimator (between the SM and P1) for different values of $d$ and $R$, based on simulations performed at various noise levels. Results suggest that as the number of replications increases, the SM estimator begins to outperform the P1 estimator for a fixed value of $d$.

\begin{figure}[]
\centering
\includegraphics[width=0.7\columnwidth]{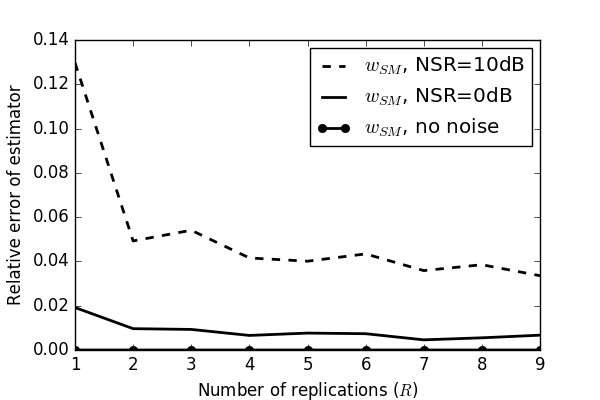}
\caption{An illustration of the accuracy of the standard SM estimator for various noise levels, as the number of replications increases. Here, $d=4$, $n=1000$, $X \sim N(1,1)$, and $w_0$ is fixed to $[1, 1, 1, 1]$. The relative error is defined the same as in previous figures, and the figure shows the mean error over 10 iterations. Here, the noise-to-signal ratio (NSR) is defined as $20 \log_{10}{(\sigma_E/|w_0|_2)}$. Even at relatively high noise levels, increasing the number of replications significantly reduces the error.}
\label{fig:figure3}
\end{figure}

\begin{figure}[]
\centering
\includegraphics[width=0.7\columnwidth]{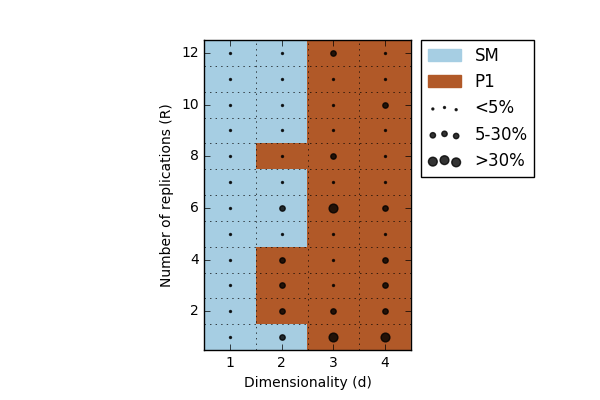}
\caption{For different values of the dimensionality $d$ and number of replications $R$, the estimator with the lower inference error (SM or P1) is shown. Specifically, at each operating point, 5 simulations were run, with the SNR set to 15 dB and the number of points in each replication set to 100. The magnitude of the error is indicated by the size of the black dot at the center of each cell. The LS estimator performed similarly to the P1 estimator (see Appendix \ref{appendix:regimes} for more details).}
\label{fig:figure2-2}
\end{figure}

\subsection{Noise Characteristics}
\label{subsection:noise}

In Section \ref{section:smd2}, we mentioned that it is possible to make the SM estimator consistent for $d=2$ if the moments of the noise have been characterized. Specifically, we include the variance of the noise in second-moment constraint of the estimator, as follows: 

\begin{equation}
\label{eqn_na_sm_d2}
\frac{1}{n} \sum_{i}^n \left( (x_i \hat{w})^2 + \sigma_E^2 \right) = \frac{1}{n} \sum_{i}^n y_i^2, 
\end{equation}

We can do a similar modification of the loss function when $d>2$, by considering the contribution of each moment of the noise source, $E$, to each sample moment of $(\boldsymbol{x} \cdot \hat{w})$. Here, $E$ is the random variable that models the noise. It can be shown (see Appendix \ref{appendix:noise}) the result is $D$ constraints of the following form, for $1 \le k \le K$:

\begin{equation}
\label{}
\underbrace{\frac{1}{n}\sum_j^k \sum_{i}^n \binom{k}{j} ( x_i \hat{w})^j \E[E^{k-j}] 
}_{M_k} = \underbrace{\frac{1}{n}\sum_{i}^n y_i^k}_{N_k} 
\end{equation}

When are the moments of the noise source known? In some cases, the characteristics of noise may be known \textit{a priori}, but in other cases, the additive noise in a system can be characterized by a carefully designed input sample. As a simple example, by running an experiment with all identical samples, one may look at the sample moments of the labels to infer the moments of $E$.

By adjusting for noise, one can significantly improve the performance of the standard SM estimator (as well as some hybrid estimators, like P2), especially when the noise levels are on the same order of magnitude as the signal in our system. This is shown in Figure \ref{fig:figure4}, where the relative error of the both the standard and noise-adjusted self-moments (NA-SM) estimator are plotted as a function of the noise-to-signal ratio in the system.

\begin{figure}[]
\centering
\includegraphics[width=0.7\columnwidth]{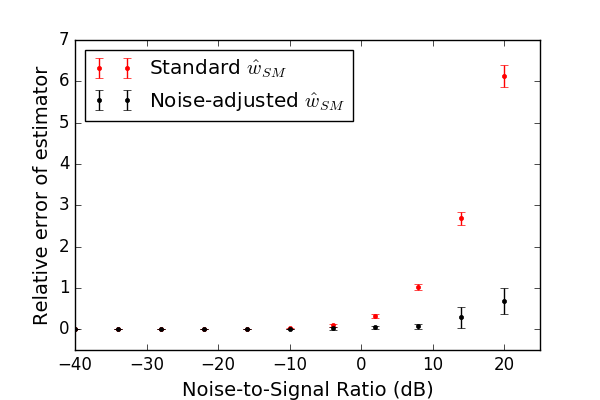}
\caption{Here, we compare the performance of the SM and NA-SM estimators, in the case of $d=2$, $n=1000$, and $w_0$ fixed to $[1, -1]$. NSR and relative error are defined the same as in previous figures. The NA-SM and SM estimators have similar relative errors, until the magnitude of the noise approximates that of the signal (at an NSR of 0 dB). At higher noise levels, the NA-SM estimator preforms significantly better than the SM estimator.}
\label{fig:figure4}
\end{figure}

\subsection{Sparsity and Regularization}
\label{subsection:regularization}

In many real-world datasets, there are features in $\boldsymbol{x}$ that do not contribute to $y$; in other words, there is sparsity in the weight vector, $w_0$. Just as in classical linear regression, to prevent overfitting these extra weights, regularization can be added to the model, at the expense of making the estimator more biased \cite{cucker2002best}.

As such, when $w_0$ is suspected to be sparse, it helps to impose $L_1$ or $L_2$ regularization to improve the estimation of the weights. The cost function can be simply adjusted by the addition of a norm-squared term to include regularization. For example, the loss function for the SM estimator in (\ref{eqn:cluster_loss}) can be adjusted to include $L_2$ regularization:

\begin{equation}
\label{eqn_robust_loss}
L(\boldsymbol{x},y,w) = \lambda_2|w|^2_2 + \sum_{r=1}^R \sum_{k=1}^K f(k)(M_{r,k} - N_{r,k})^2
\end{equation}

Similar adjustments can be made to the LS and hybrid estimators. Fig. \ref{fig:regularization} shows that $L_2$ regularization can significantly improve error estimation, particularly as the sparsity of $w_0$ increases. When $w_0$ is not sparse, however, regularization has a slightly detrimental effect, as it increases the bias of the estimator.

\begin{figure}[t]
\centering
\includegraphics[width=0.7\textwidth]{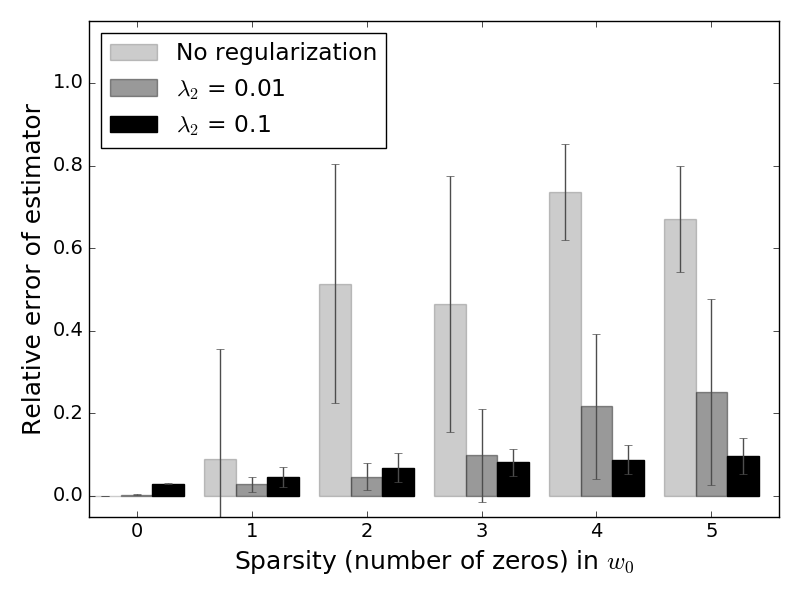}
\caption{This figure shows the effect of different levels of $L_2$ regularization on relative error. Here, $n=1000$, and $w_0$ consists of 2 weights of $1$ and $1$, while the rest of the weights are zero. As such, the dimensionality $d$ is 2 more than the sparsity in each case. The SM estimator with regularization parameter $\lambda_2$ is used to calculate $\hat{w}$. As the sparsity (and number of features) increase, the relative error increases, but less so when $\lambda_2=0.01$, and significantly less so when $\lambda_2$ is increased to $0.1$.}
\label{fig:regularization}
\end{figure}

\subsection{Numerical Algorithm for Robust Shuffled Regression}
\label{subsection:algorithm}

In previous subsections, we have introduced loss functions for each of the SM, LS, and hybrid estimators, and we have modified these loss functions in specific ways to incorporate regularization, partial ordering information, and knowledge of the noise source. The result is, for each estimator, a loss function in the weights, $w$. To solve for optimal weights, we simply minimize the particular loss function -- e.g. (\ref{eqn_LS_simplified}) or (\ref{eqn_cost_function}). 

Since the functions are, in general, non-convex, we use a multi-start strategy with gradient descent. Algorithm \ref{alg:example}, introduced in Section \ref{section:intro}, presents a numerical algorithm for shuffled regression that can be used for robust inference on real-world datasets. Variations of this algorithm (e.g. with or without replications) were used to perform the experiments in Section \ref{section:experiments}.

\section{Experiments}
\label{section:experiments}

We conducted two experiments to demonstrate the feasibility of shuffled regression on real-world datasets. In the first experiment, we applied shuffled regression to standard regression datasets with dimensionality ranging from $2 \le d \le 9$. We randomly permuted the labels in these datasets to simulate shuffled data, and measured the relative error between the weights we estimated on the shuffled data and those weights estimated by ordinary least squares before shuffling the labels. 

The second experiment is based on a potential application of shuffled regression. We explored whether shuffled linear regression can be used to infer the \textit{qualitative} effect of different nucleotide bases and motifs (sequences of bases) on the affinity of a DNA sequence to a particular target. The dataset consisted of the affinities of 1200 DNA sequences, measured across 5 independent experiments. This experiment was in a higher-dimensional setting ($d=21$). All experimental datasets can be found in the associated GitHub repository (see end of Section \ref{section:intro}).

\subsection{Standard Datasets}
\label{subsection:standard}

We applied the estimators developed in the previous sections to seven regression datasets from the UCI repsoitory\footnote{\url{
https://archive.ics.uci.edu/ml/}} and MATLAB\footnote{\url{https://www.mathworks.com/help/stats/_bq9uxn4.html}}. The datasets were chosen to have real-valued features and most had $R^2$ values of at least 0.8, which provided some confidence that the data was generated from a linear source, albeit one that was noisy. Three of the seven datasets were modified to remove outliers until the $R^2$ value reached a value of at least 0.8. Additionally, six synthetic datasets were generated separately to provide further validation of the method. All of the datasets were normalized so that $y$ and each column of $\boldsymbol{x}$ ranged between 0 and 1 (except the bias column, which was identically 1). Properties of all 13 normalized datasets are listed in Table \ref{table1}. 

Following \cite{yu2013propto}, we explored the effect of replications by randomly splitting the data into 1, 2, 4, 6, and 8 replications, and within each replication, shuffling the labels ourselves to simulate a weakly-supervised setting. We then used algorithm $\ref{alg:example}$ (without regularization or adjusting for noise) to estimate weights, $\hat{w}$. The choice of estimator depended on the values of $d$ and $R$, based on our earlier analysis of the optimal regimes of operation. Specifically, we used the SM estimator when $d=1,2$ or when $R \ge 3d$. Otherwise, we used the P1 estimator.

Finally, we computed the relative error by comparing the estimated weights to the weights inferred by ordinary least-squares linear regression (applied to the unshuffled data). The relative error is defined as $\lvert \hat{w} - \hat{w}_{OLS} \rvert_2/\lvert \hat{w}_{OLS} \rvert_2$. These results, averaged across 10 trials of shuffled regression, are shown in Table \ref{table2}. 

To provide a baseline against which to compare these results, we also ran standard linear regression with a random permutation of the labels within each replication. This baseline ensured that we were not partitioning the data into too many replications, effectively disambiguating the ordering of the labels. The results of the baseline experiment are found in Table \ref{table-nc} in Appendix \ref{appendix:negative-control}, and indicate that the ordering of the labels remains ambiguous, even for $R=8$, across all datasets. 

\begin{table}[]

\caption{This table lists the names and basic properties of the datasets that were used to validate shuffled linear regression in the experiment described in Section \ref{subsection:standard}. The name of the dataset corresponds to the name found in either the UCI or MATLAB repository. The $R^2$ value is goodness-of-fit that results when ordinary linear regression is applied to the unshuffled labels. $\hat{\sigma}_E$ is an estimate of the standard deviation of the noise source, computing by taking the standard deviation of the residuals. (*Outliers were removed from the three datasets marked with an asterisk to improve linearity.)} \label{table1}
  \vspace{3mm}

\begin{center}
 \begin{tabu} to 0.9\columnwidth { | X[l]  X[c]  X[c]  X[c]  X[c] |}

 \hline
 Dataset &  Features ($d$) & Size ($n$) & $R^2$
 & $\hat{\sigma}_E$
 \\ 
 \hline\hline

lsat & 2 & 15 & 0.6 & 0.21 \\
\hline
accidents & 2 & 51 & 0.84 & 0.08 \\
\hline
acetylene & 4 & 16 & 0.92 & 0.09 \\
\hline
power plant & 5 & 9568 & 0.93 & 0.06 \\
\hline
airfoil$^*$ & 6 & 1043 & 0.83 & 0.06 \\
\hline
yacht$^*$ & 7 & 134 & 0.88 & 0.05 \\
\hline
concrete$^*$ & 9 & 559 & 0.91 & 0.05 \\
\hline

\hline
synthetic1 & 3& 50 & 0.93 & 0.06 \\ \hline
synthetic2 & 4& 100& 0.99&0.02\\ \hline
synthetic3 & 5& 200& 0.90&0.05\\ \hline
synthetic4 & 6& 400& 0.98&0.03\\ \hline
synthetic5 & 7& 400& 0.94&0.04\\ \hline
synthetic6 & 7& 800& 0.81&0.07\\ \hline
  
\end{tabu}
\end{center}
\end{table}

\begin{table}[h]

\caption{This table represents the results from applying the shuffled regression  estimator to the data. Each cell represents the relative error, as defined earlier, between the OLS weights and the estimated weights. The mean (over 10 trials) is reported. Errors of less than 0.3 are shaded in green, between 0.3 and 1.0 in yellow, and above 1.0 in red. } \label{table2}
  \vspace{3mm}
  \centering

 \begin{tabu} to 0.9\columnwidth { | p{2cm} | X[c] | X[c] | X[c] | X[c] | X[c]|}
    

\cline{2-6}
\multicolumn{1}{c}{} & \multicolumn{5}{|c|}{\textbf{Number of Replications} ($R$) } \\

    \hline
    \textbf{Dataset} & 1 & 2 & 4 & 6 & 8\\\hline 
    
lsat & \cellcolor{yellow!25} 0.31 & \cellcolor{green!15} 0.20 & \cellcolor{yellow!25} 0.54 
& \cellcolor{green!15} 0.27 & \cellcolor{green!15} 0.16 
\\
\hline
accidents & \cellcolor{green!15} 0.09 & \cellcolor{green!15} 0.04 & \cellcolor{green!15} 0.03 & \cellcolor{green!15} 0.13 & \cellcolor{green!15} 0.21 \\
\hline
acetylene & \cellcolor{red!15} 1.58 & \cellcolor{yellow!25} 0.66 & \cellcolor{green!15} 0.25 & \cellcolor{green!15} 0.11 & \cellcolor{green!15} 0.04 \\
\hline
power plant & \cellcolor{yellow!25} 0.34 & \cellcolor{yellow!25} 0.33 & \cellcolor{yellow!25} 0.33 & \cellcolor{yellow!25} 0.32 & \cellcolor{yellow!25} 0.33 \\
\hline
airfoil & \cellcolor{yellow!25} 0.69 & \cellcolor{yellow!25} 0.40 & \cellcolor{yellow!25} 0.36 & \cellcolor{green!15} 0.20 & \cellcolor{green!15} 0.15 \\
\hline
yacht & \cellcolor{red!15} 10.00 & \cellcolor{red!15} 8.11 & \cellcolor{red!15} 1.15 & \cellcolor{green!15} 0.09 & \cellcolor{green!15} 0.11 \\
\hline
concrete & \cellcolor{red!15} 1.20 & \cellcolor{yellow!25} 0.84 & \cellcolor{yellow!25} 0.66 & \cellcolor{yellow!25} 0.66 & \cellcolor{yellow!25} 0.47 \\
\hline
\hline

synthetic1 & \cellcolor{green!15} 0.22 & \cellcolor{green!15} 0.14 & \cellcolor{green!15} 0.05 & \cellcolor{green!15} 0.04 & \cellcolor{green!15} 0.03 \\ \hline
synthetic2 & \cellcolor{green!15} 0.17 & \cellcolor{green!15} 0.02 & \cellcolor{green!15} 0.02 & \cellcolor{green!15} 0.02 & \cellcolor{green!15} 0.01 \\\hline
synthetic3 & \cellcolor{red!15} 1.24 & \cellcolor{yellow!25} 0.32 & \cellcolor{green!15} 0.12 & \cellcolor{green!15} 0.1 & \cellcolor{green!15} 0.08 \\\hline
synthetic4 & \cellcolor{red!15} 1.45 & \cellcolor{green!15} 0.13 & \cellcolor{green!15} 0.05 & \cellcolor{green!15} 0.05 & \cellcolor{green!15} 0.05 \\\hline
synthetic5 & \cellcolor{red!15} 1.45 & \cellcolor{yellow!25} 1.0 & \cellcolor{green!15} 0.24 & \cellcolor{green!15} 0.12 & \cellcolor{green!15} 0.13 \\\hline
synthetic6 & \cellcolor{red!15} 1.21 & \cellcolor{yellow!25} 0.95 & \cellcolor{yellow!25} 0.46 & \cellcolor{green!15} 0.27 & \cellcolor{green!15} 0.21 \\\hline

\end{tabu}

\end{table}

\subsection{Case Study: Aptamer Evolution}
\label{subsection:facs}

\begin{table}[t]

\begin{center}

  \caption{This table presents the top 5 motifs correlated  with increasing and decreasing binding affinity, as predicted by ordinary least-squares (unshuffled labels) and the LS estimator (shuffled labels). The motifs that are highlighted in green are common predictions of both techniques. Of the top 5 motifs predicted by OLS, 3 were also predicted by LS. All 5 of the bottom 5 motifs were also predicted LS. The Spearman rank-order correlation between the weights assigned to all 20 features by both methods is also included.} \label{table3}
  
  \vspace{3mm}

 \begin{tabu} to 0.5\textwidth { | X[c]  | X[c] | X[c] |}
 \hline
 \textbf{Rank} & $\hat{w}_{\mathrm{OLS}}$ & $\hat{w}_{\mathrm{LS}}$ \\ 
 \hline\hline
 
	1& \cellcolor{green!15}AT	& \cellcolor{green!15}AT	 \\
\hline
	2 & TGT	& GAT	 \\
\hline
3&	G&	AAT \\
\hline
4	& \cellcolor{green!15}ATC	& \cellcolor{green!15}ATC	 \\
\hline
5	& \cellcolor{green!15}CGT	& \cellcolor{green!15}CGT\\
\hline
\hline
\hline
16	& \cellcolor{green!15}GT	& \cellcolor{green!15} GT\\
\hline
17	& \cellcolor{green!15} TAT	& \cellcolor{green!15} AAT	 \\
\hline
 
18	& 	\cellcolor{green!15}AAT & \cellcolor{green!15}	GAT \\
\hline
19	& \cellcolor{green!15}GAT	&	\cellcolor{green!15}TAT \\
\hline
 20	& 	\cellcolor{green!15} CAT & \cellcolor{green!15} CAT	 \\
 
 \hline
 
\multicolumn{3}{|c|}{Rank correlation of all 20 features: \textbf{+0.893}} \\

\hline
 
\end{tabu}
\end{center}
\end{table}

Here, we consider a potential, real-world application of shuffled regression: \textit{aptamer evolution}. Aptamers are short DNA sequences that bind to target molecules. Aptamer evolution refers to the process of using chemical and computational methods to design and select aptamers with an affinity to a particular target \cite{knight2009array}. Computational techniques generally begin by measuring the affinity of a large number of DNA sequences, and then identifying \textit{motifs}, nucleotide bases or sequences of bases (such as ``ATC'' or ``GG''), that are common to many sequences which bind to the desired target molecule. Once these motifs are identified, new aptamers can be synthesized that combine various motifs or prominently feature the relevant motifs.

A variety of techniques can be used to make the initial measurement of affinities. The public dataset that we are using consists of measurements made through microarray analysis \cite{knight2009array}. In this method, the affinity is measured for each aptamer individually, but there are other techniques, where the same affinity measurement can be made on a population level. For example, fluorescence-activated cell sorting (FACS) can be used to measure the affinity of thousands of aptamers at once \cite{shapiro2005practical}. FACS can be faster and cheaper than microarray analysis, but unlike microarrays, it does not provides affinities for individual sequences -- it instead returns a histogram of affinities across all particles. Thus, we can simulate a FACS measurement by shuffling the labels from the microarray analysis dataset.

Once the initial measurements have been made, the question of identifying which motifs are responsible for affecting affinity becomes an inference problem, which can be approached through regression. So, in this experiment, our goal was to determine if shuffled linear regression applied to shuffled labels (coming from a simulated FACS experiment) identified the same motifs as standard linear regression applied to unshuffled labels (coming from the microarray experiment). 

The dataset we used was a pre-processed version of the test dataset in \cite{knight2009array}. We first reduced the dimensionality of the dataset by choosing the 20 most significant motifs (as computed by an initial run of standard linear regression) as features. We also restricted the dataset to the top 1200 aptamers, to create a more homogeneous and linear dataset.Thus, we were left with a dataset of size $n=1200$ and $d=21$, including the bias dimension.

First, standard linear regression was applied to  to the unshuffled labels. Then, the dataset was randomly partitioned into 5 replications, representing 5 individual FACS experiments. Then, the LS estimator was applied to the shuffled labels. The top 5 weights (motifs predicted to increase affinity) and the bottom 5 weights (motifs predicted to decrease affinity) from both techniques were recorded and are compared in Table \ref{table3}.

\section{Discussion}
\label{section:discussion}

In this work, we have proposed a framework for preforming shuffled linear regression that is computationally tractable and approximately recovers the weights of a noisy linear system. This framework is based on several estimators, including analogs of the classical least-squares estimator and the method-of-moments estimators, which we denote as the LS and SM estimators, respectively. We investigated some statistical properties of each estimator; for example, we presented theoretical results that show that the SM estimator is consistent for $d=1$ while the LS estimator is not (see Theorems \ref{thm_LS_inconsistency} and \ref{thm_sm_consistency}).

For higher dimensions, we presented empirical results that identified the regimes (in terms of $n$, $d$, and $R$) in which each estimator excels (see Figures \ref{fig:regimes1} and \ref{fig:figure2-2}). We found that in for $d =1,2$, the SM estimator generally provided estimates with the lowest inference error. In higher dimensions, the hybrid P1 estimator, formed by projecting the input features to a lower dimension before applying the SM estimator, yielded better estimates than either the LS or SM estimators alone. A thorough theoretical analysis that supports these results remains open for future work, as does the question of whether these estimators are optimal among all order-invariant estimators (see Appendix \ref{appendix:order-invariant}). 

We then applied our algorithm to standard and synthetic datasets, of varying size and dimensionality. In most cases, the weights that we estimated by shuffled linear regression were similar to those inferred by classical linear regression (applied to unshuffled labels). This was especially true when the dimensionality of the dataset was low or there were many replications of the experiment. In almost every case, as the number of replications increased to $R=8$, the relative error between the weights dropped to 20\% or even less. In most datasets, the relative error dropped significantly even with just 2 or 4 replications.

There were a few exceptions to this trend, such as the \texttt{power\char`_plant} dataset, where the relative error remained roughly constant even as the number of replications increased. Analysis of the residuals suggested that in this dataset, the data did not generate from a linear model. Thus, shuffled linear regression likely \textit{overfit}, because it found a permutation of the labels that produced a better linear fit than the ``true'' weights. 

To simulate replications, we divided our dataset into clusters randomly; in other words, our replications were simulating independent and \textit{identical} experiments. It is likely that further improvements in accuracy can result if the experiment is designed so that each replication consists of intentionally different input data (e.g. different replications may have slightly different means, variances, or ranges).

In general, we found that partial ordering information significantly reduces the inference error in shuffled regression. Comparing these results to that of a baseline estimate (see Appendix \ref{appendix:negative-control}) shows that increasing the number of replications did not completely disambiguate the ordering information; instead, it likely reduced the variance of the estimators, in turn reducing the inference error.

Furthermore, our case study with aptamer evolution provided evidence that even when the dimensionality is high, the estimators that we have described may be useful to determine \textit{qualitative} significance of different input features. From a dataset of 1200 aptamers and 20 motifs, our approach recovered 3 of the top 5 motifs correlated with increasing binding affinity, and all 5 of the top 5 motifs correlated with decreasing binding affinity.

Throughout our work, we have assumed that the data is generated from a noisy linear model. This raises several questions: from an information-theoretic perspective, does this assumption provide similar information to that would otherwise be provided by the mutual ordering of the input features and labels in the classical setting of linear regression? Does this assumption restrict our ability to generalize shuffled regression to non-linear settings, such as kernel regression? Can we use the recovered weights as a means to estimate the unknown permutation matrix, as an alternative to techniques presented in \cite{pananjady2016linear}? We hope to explore these questions in future work.

We have also released our code (see end of Section \ref{section:intro}) so that other researchers can investigate the effectiveness of shuffled linear regression on their own datasets.

\section{Conclusion}
\label{section:conclusion}

We have proposed a mathematical and algorithmic framework for shuffled linear regression, a technique to perform inference on datasets where the ordering of the labels is randomly permuted with respect to the input features. Our experiments on synthetic and standard datasets provide empirical evidence to show that shuffled linear regression estimates the weights of a linear model quantitatively and qualitatively. These results raise many important questions, while demonstrating that shuffled linear regression can be of practical interest in certain domains and datasets, which were formerly not open to analysis.

\bibliographystyle{unsrt}
\bibliography{references}

\begin{appendices}

\setcounter{thm}{0}

\section{Proof of Major Theorems} \label{appendix:proofs}

\subsection{Proof of Theorem 1}
\label{appendix:proof1}

\begin{thm}
\label{app_thm_LS_inconsistency}
Let $\boldsymbol{x} \in  \mathbb{R}^{n \times 1}$ be sampled from a Gaussian random variable with mean $\mu_X$ and variance $\sigma_X^2$. Let $y = \boldsymbol{\pi_0 x} w_0 + e$ be the product of $\boldsymbol{x}$ with an unknown scalar weight $w_0$ and an unknown $n \times n$ permutation matrix $\boldsymbol{\pi_0}$ added to zero-mean Gaussian noise with variance $\sigma_E^2$. The least-squares estimate of $w_0$, as defined in (\ref{eqn_LS_shuffled}), is inconsistent. In fact, the estimator converges to the following limit with probability 1:
\begin{equation*}
\lim_{n \to \infty} \hat{w}_{\mathrm{LS}}  = w_0 \left( \frac{\mu_X^2 + \sigma_X \sqrt{\sigma_X^2 + \frac{\sigma_{E}^2}{w_0^2}}}{\mu_X^2 + \sigma_X^2} \right) \end{equation*}
\end{thm}
\begin{proof}

$\,$
We establish this proof through two lemmas that we prove at the end of appendix \ref{appendix:proof1}. Firstly:
$\,$

\begin{lemma}
\label{lemma:simplerls}
In the case of $d=1$ and $w > 0$, the LS estimator can be written as follows:

$$ \hat{w}_{LS} = \arg \min_{w} \, \lvert y^{\uparrow} -  x^{\uparrow} \cdot w \rvert^2  $$

where we use the notation $v^{\uparrow}$ to denote an $n$-by-$1$ vector that consists of the $n$ entries of $v$ sorted to be in ascending order.
\end{lemma}
(We have assumed for simplicity $w > 0$, but if $w < 0$, an an analogous argument can be made to show that the least-square difference occurs when $x$ is sorted in \textit{descending} order, which leads to the same conclusion as below.) Using Lemma 1, we can treat $y^{\uparrow}$ and  $x^{\uparrow}$ as two new vectors to which we are applying ordinary least-squares (OLS) regression. Using well-known results from the OLS estimator, we can write a closed-form expression for the LS estimator as:

$$ \left( \frac{1}{n} x^{\uparrow T} x^{\uparrow} \right)^{-1} \left( \frac{1}{n} x^{\uparrow T} y^{\uparrow} \right) = \left( \frac{1}{n} \sum_{i}^{n} x_{(i)}^2 \right)^{-1} \left( \frac{1}{n} \sum_{i}^{n} x_{(i)} y_{(i)} \right)  $$
where we have adopted the notation from order statistics that $v_{(i)}$ refers to the $i^{th}$ smallest entry in the vector $v$.
We are interested in finding the limit of:
$$ \lim_{n \rightarrow \infty} \left( \frac{1}{n} \sum_{i}^{n} x_{(i)}^2 \right)^{-1} \left( \frac{1}{n} \sum_{i}^{n} x_{(i)} y_{(i)} \right)  $$
By Slutsky's Theorem \cite{arthurs.goldberger1964}, we can separate the two terms and find their limits separately.
Quite clearly, $$\lim_{n \rightarrow \infty} \left( \frac{1}{n} \sum_{i}^{n} x_{(i)}^2 \right)^{-1} = \lim_{n \rightarrow \infty} \left( \frac{1}{n} \sum_{i}^{n} x_{i}^2 \right)^{-1} = E[X^2]^{-1} = \frac{1}{\mu_X^2 + \sigma_X^2} $$

The second term requires more work. We will find its limit using Lemma 2.
$\,$

\begin{lemma}
If the entries in a vector $x$ are drawn from Gaussian-distributed random variable $X \sim N(\mu_X, \sigma_X^2)$ and the entries in a vector $y$ are drawn from Gaussian-distributed random variable $Y\sim N(\mu_Y, \sigma_Y^2)$, 

$$ \lim_{n \rightarrow \infty} \left( \frac{1}{n} \sum_{i}^{n} x_{(i)} y_{(i)} \right) \rightarrow \mu_X \mu_Y + \sigma_X \sigma_Y $$
\end{lemma}

$\,$
Because we have assumed that the entries of $x$ are i.i.d. drawn from a random normal distribution, we can treat each entry $x_{1}, x_{2} ...$ as the outcome of an i.i.d. Gaussian. $X_{1}, X_{2} ... \sim N(\mu_X, \sigma_X^2)$. Similarly, each $y_1, y_2, ...$ is the outcome of random variables $Y_1, Y_2, ... \sim N(w_0 \cdot \mu_X, (w_0 \cdot \sigma_X)^2 + \sigma_{E}^2)$. The distribution of $Y_{i}$ follows directly from the way the vector $y$ is generated -- by multiplying $x$ by a weight $w_0$ and adding a Gaussian noise with variance $\sigma_{E}^2$.
Thus, in our case,

$$ \lim_{n \rightarrow \infty} \left( \frac{1}{n} \sum_{i}^{n} x_{(i)} y_{(i)} \right) \rightarrow   w_0 \mu_X^2 + \sigma_X \sqrt{w_0^2 \sigma_X^2 + \sigma_{E}^2}$$
Using the continuous mapping function, we can combine these results to show that the LS estimator converges to the following limit, proving our theorem.
$$ \lim_{n \rightarrow \infty} \left( \frac{1}{n} \sum_{i}^{n} x_{(i)}^2 \right)^{-1} \left( \frac{1}{n} \sum_{i}^{n} x_{(i)} y_{(i)} \right) = \lim_{n \rightarrow \infty} \hat{w}_{LS} \rightarrow w_0 \left( \frac{\mu_X^2 + \sigma_X \sqrt{\sigma_X^2 + \frac{\sigma_{E}^2}{w_0^2}}}{\mu_X^2 + \sigma_X^2} \right)$$
$\,$

\end{proof}

\subsection*{Proof of Lemma 1}

\begin{proof}
We claim that the following objective functions are equivalent when $d=1$ and $w > 0$.
$$\arg \min_{w} \min_{\boldsymbol{\pi}} \, \lvert y - \boldsymbol{\pi} \cdot x \cdot w \rvert^2 $$
$$\arg \min_{w} \, \lvert y^{\uparrow} -   x^{\uparrow} \cdot w \rvert^2  $$
In other words, if we assume without loss of generality that the entries in $y$ are already in ascending order, we would like to prove that the ordering of $x$ that minimizes the squared difference between $y$ and $x \cdot w$ is when $x$ is also sorted in ascending order. 

First of all, note that since $w>0$, it does not affect the ordering of $x \cdot w$, so we can absorb it into $x$ and show only that $\lvert y^{\uparrow} - \boldsymbol{\pi} \cdot x \rvert^2$ takes a minimum when $x$ is also sorted in ascending order. 

Now, for the sake of contradiction, let us assume that $\lvert y^{\uparrow} - \boldsymbol{\pi} \cdot x \rvert^2$ is minimized for a permutation of $x$, call it $x'$, that does not have its entries in ascending order. Then, there must be $i, j$ such that $i < j$ but $x'_i > x'_j$ where we use the notation $v_i$ to refer to the $i^{th}$ element of $v$. 
But then:

$$\lvert y^{\uparrow} - x' \rvert^2 = $$
$$\sum_k^n (y^{\uparrow}_k - x'_k )^2 = $$
$$(y^{\uparrow}_{i} - x'_i)^2 + (y^{\uparrow}_{j} - x'_j)^2 +  \sum_{k \ne i,j}^n (y^{\uparrow}_{k } - x'_k )^2 = $$
$$(y^{\uparrow}_{i} - x'_i)^2 + (y^{\uparrow}_{j} - x'_j )^2 +  \sum_{k \ne i,j}^n (y^{\uparrow}_{k } - x'_k )^2 \ge $$
$$(y^{\uparrow}_{j} - x'_i)^2 + (y^{\uparrow}_{i} - x'_j )^2 +  \sum_{k \ne i,j}^n (y^{\uparrow}_{k } - x'_k )^2$$
But this is a contradiction, because we have showed that we can achieve a smaller LS by switching the $i^{th}$ and $j^{th}$ entry in $x'$.
\textit{Note}: The inequality in the last step follows from the fact that the difference between the final expression and the one before it can be written as:

$$ -2(x'_i y^{\uparrow}_i + x'_j y^{\uparrow}_j) + 2(x'_i y^{\uparrow}_j + x'_j y^{\uparrow}_i) = $$

$$ 2 (y^{\uparrow}_j - y^{\uparrow}_i) (x'_i - x'_j)  \ge 0 $$
The last equality is greater than or equal to zero because  $(y^{\uparrow}_j - y^{\uparrow}_i) \ge 0$ because $y^{\uparrow}$ is sorted in ascending order and $(x'_i - x'_j) \ge 0$ by assumption.
$\,$
\end{proof}

\subsection*{Proof of Lemma 2}

We claim that if the entries in a vector $x$ are drawn from Gaussian-distributed random variable $X \sim N(\mu_X, \sigma_X^2)$ and the entries in a vector $y$ are drawn from Gaussian-distributed random variable $Y\sim N(\mu_Y, \sigma_Y^2)$, then with probability 1,

$$ \lim_{n \rightarrow \infty} \left( \frac{1}{n} \sum_{i}^{n} x_{(i)} y_{(i)} \right) \rightarrow \mu_X \mu_Y + \sigma_X \sigma_Y $$
To prove this, we start by defining $Z \sim N(0,1)$ to be the standard Gaussian random variable. Let us consider the probability distribution of the order statistics of a vector whose entries are sampled from $Z$.
If $z$ is an $n$-dimensional vector whose entries are sampled from $Z$, then, we write the marginal distributions of $z_{(1)}, z_{(2)} ...  z_{(n)}$ as $Z_{(1)}, Z_{(2)}, ... Z_{(n)}$
$\,$
\begin{prop}
In the limit of large $n$, the variance of each $Z_{(k)}$ approaches 0. 
\end{prop}

$\,$

\begin{proof}
It is known that the variance of each $Z_{(k)}$ is bounded. For example, from Proposition 4.2 in \cite{Boucheron2012}, 

$$ Var[Z_{(k)}] \le \frac{C}{k \log{\frac{2n}{k}}  - k \log{(1+\frac{4}{k} \log{\log{\frac{2n}{k}}})}}$$

for a constant $C$, where $1 \le k \le \frac{n}{2}$. By symmetry, it is clear that the variance can be written similarly for $\frac{n}{2} \le k \le n$.

Now, observe that the right-hand side of the inequality itself is bounded for large enough $n$ as follows:

\begin{equation*}
\begin{split}
 \frac{C}{k \log{\frac{2n}{k}}  - k \log{(1+\frac{4}{k} \log{\log{\frac{2n}{k}}})}} \le \\ \frac{C}{k \log{\frac{2n}{k}} - 4 \log{\log{\frac{2n}{k}}}} \le \\ \frac{C}{\log{2n} - 4 \log{\log{2n}}} \le \\ \frac{C}{\log{\log{2n}}}
 \end{split}
 \end{equation*}
 
This is because in each inequality, we have decreased the magnitude of the denominator, which increases the value of the fraction, as long as the denominator remains greater than zero, which we argue below is the case for large enough $n$.

\begin{itemize}
\item \textbf{Inequality 1}: This follows because $x \ge \log{(1+x)} \, \forall \, x \ge 0$ and $\frac{4}{k} \log{\log{\frac{2n}{k}}} \ge 0$, since the smallest value it can take is when $k=\frac{n}{2}$, when it evaluates to $\frac{8}{n} \log \log 4 > 0$.

We also need to ensure that the new denominator remains greater than zero. The new denominator is $k \log{\frac{2n}{k}} - 4\log{\log{\frac{2n}{k}}}$. By taking the derivative of the first term, it is easy to show that it is minimized for $k=1$ and the second term is maximized for $k=1$, and thus difference is largest when $k=1$. Substituting, we get that the difference is $\log{2n} - 4 \log{\log{2n}}$, which can straightforwardly be shown to be greater than zero for large enough $n$.

\item \textbf{Inequality 2}: This inequality follows from the preceding analysis, namely that the value of the denominator is largest when $k=1$. We have also already confirmed that the resulting denominator is positive for large enough $n$.

\item \textbf{Inequality 3}: This inequality follows from the fact that for large enough $n$,

$$\log{2n} - 5 \log{\log{2n}} > 0 \implies \log{2n} - 4 \log{\log{2n}} > \log{\log{2n}}$$

It is clear that this inequality is true for large enough $n$ and it is clear that for $n$ this large or larger, the final denominator is positive.

\end{itemize}

Similar reasoning can be applied in the case that $\frac{n}{2} \le k \le n$. Thus, this completes the proof of Proposition 1. 

\vspace{5mm}

By immediate application of Chebyshev's inequality, this means that $\forall \delta > 0$,

$$\lim_{n \rightarrow \infty} \Pr{ \lvert z_{(i)}-E[Z_{(i)}] \rvert} > \delta = 0$$
Furthermore, if $z'$ is another vector whose elements are sampled from the random distribution, $Z$ then by the triangle inequalty, we have:
$$\lim_{n \rightarrow \infty} \Pr{ \lvert z_{(i)}-z'_{(i)} \rvert} > \delta = 0$$
Now, let us return to the limit we are finding. Clearly, 

$$X \sim (\sigma_X Z + \mu_X) \,\, , \,\, Y \sim  (\sigma_Y Z + \mu_Y) $$
Because neither multiplication by a positive constant nor adding a scalar affects the ordering of elements in a vector, we can rewrite the limit as:
$$\lim_{n \rightarrow \infty} \left( \frac{1}{n} \sum_{i}^{n} x_{(i)} y_{(i)} \right) =  \lim_{n \rightarrow \infty} \left( \frac{1}{n} \sum_{i}^{n} (\sigma_X z_{(i)} + \mu_X) (\sigma_Y z'_{(i)}  + \mu_Y) \right)$$ 

$$= \lim_{n \rightarrow \infty} \left( \frac{1}{n} \sum_{i}^{n} (\sigma_X \sigma_Y z_{(i)} z'_{(i)} + \mu_Y \sigma_X z_{(i)} + \mu_X \sigma_Y z'_{(i)}  + \mu_X \mu_Y ) \right) $$

$$= \mu_X \mu_Y + (\mu_Y \sigma_X + \mu_X \sigma_Y) E[Z] + \sigma_X \sigma_Y \lim_{n \rightarrow \infty} \left( \frac{1}{n} \sum_{i}^{n}  z_{(i)} z'_{(i)}  \right) $$

$$= \mu_X \mu_Y  + \sigma_X \sigma_Y \lim_{n \rightarrow \infty} \left( \frac{1}{n} \sum_{i}^{n}  z_{(i)} z'_{(i)}  \right)$$
If the limit in the last expression equals $1$, our original limit converges to the desired value. Thus, we need to show that $\lim_{n \rightarrow \infty} \left( \frac{1}{n} \sum_{i}^{n}  z_{(i)} z'_{(i)}  \right) \rightarrow 1$
First, note that:
$$\lim_{n \rightarrow \infty} \left( \frac{1}{n} \sum_{i}^{n}  z_{(i)} z'_{(i)}  \right) = $$ 

$$\lim_{n \rightarrow \infty} \left( \frac{- 1}{2n} \sum_{i}^{n} \left( [z_{(i)} - z'_{(i)}]^2 - z_{(i)}^2 - z'^2_{(i)} \right) \right) = $$
$$\lim_{n \rightarrow \infty} \left( \frac{- 1}{2n} \sum_{i}^{n} \left( [z_{(i)} - z'_{(i)}]^2 - 2 \right) \right) = $$
$$1 - \lim_{n \rightarrow \infty} \left( \frac{1}{2n} \sum_{i}^{n} \left( [z_{(i)} - z'_{(i)}]^2  \right) \right) = 1 $$

The last equality follows because, we have already showed above that $\lim_{n \rightarrow \infty} \left( \frac{1}{2n} \sum_{i}^{n} \left( [z_{(i)} - z'_{(i)}]^2  \right) \right) \rightarrow 0$ with probability 1.

\end{proof}
Thus, we conclude our proof of Lemma 2.

\subsection{Proof of Theorem 2}
\label{appendix:proofthm2}

\begin{thm}
\label{app_thm_sm_consistency}
Let $\boldsymbol{x} \in  \mathbb{R}^{n \times 1}$ be sampled from a Gaussian random variable with mean $\mu_X$ and variance $\sigma_X^2$. Let $y = \boldsymbol{\pi_0 x} w_0 + e$ be the product of $\boldsymbol{x}$ with an unknown scalar weight $w_0$ and an unknown $n \times n$ permutation matrix $\boldsymbol{\pi_0}$ added to zero-mean Gaussian noise with finite variance $\sigma_E^2$. The SM estimator, as defined in (\ref{eqn_sm_d1}), is consistent when $\mu_X \ne 0$. In other words, with probability 1,
\begin{equation*}
\lim_{n \to \infty} \hat{w}_{\mathrm{SM}}  = w_0. \end{equation*}
Furthermore, given a fixed sample vector $x$, the SM estimator is unbiased as long as as $\bar{x} \equiv \frac{1}{n}\sum_i x_i \ne 0$, meaning that:
\begin{equation*}
 \E[\hat{w}_{\mathrm{SM}} - w_o] = 0,
\end{equation*}
\end{thm}

\begin{proof}
For $d=1$, The  $\hat{w}_{SM}$ estimator takes the form:

$$\frac{\sum_i{x_i w_0 + e_i}}{\sum_i{x_i}} = \frac{w_0\sum_i{x_i  + \sum_i{e_i}}}{\sum_i{x_i}} = w_0 + \frac{\frac{1}{n}\sum_i{e_i}}{\bar{x}} $$

Since the noise is zero-mean, $\E[\sum_i{e_i}]=0$, so this estimator is unbiased.  The mean-squared error in the estimator is:

$$\E[(\hat{w}_{\mathrm{SM}} - w_o)^2] = \E\left[ \frac{( \sum_i{e_i})^2}{(\sum_i{x_i})^2} \right] = \E \left[\left( \frac{1}{n}\sum_i{e_i} \right)^2 \right] \cdot \E\left[ \frac{1}{(\frac{1}{n} \sum_i{x_i})^2} \right] = \frac{\sigma_E^2}{n} \cdot \E\left[ \frac{1}{(\frac{1}{n} \sum_i{x_i})^2} \right]$$

In the limit of large $n$, $\frac{1}{n} \sum_i{x_i} $ approaches a Gaussian random variable with mean $\mu_X$ and infinitesimal variance. Thus, the error becomes:

$$\lim_{n \to \infty} \E[(\hat{w}_{\mathrm{SM}} - w_o)^2] = \frac{\sigma_E^2}{\mu_X^2 n} =0 ,$$

under the assumption $\mu_X \ne 0$, proving consistency.

\end{proof}

\section{An analytical solution for $w_{SM}$ when $d=2$}
\label{appendix:analyticalsm2}
In this section, we derive an analytical solution for the SM estimator for $d=2$. To simplify notation, we will use the notation $\E[y]$, $\E[y^2], ...$ to refer to sample moments of a vector $y$, and $E[xy]$ to refer to the sample cross-moments of vectors $x$ and $y$. We can write equations for the first two self-moments:
\begin{equation*}
\begin{split}
 \E[\boldsymbol{x} \cdot w] = \E[y] \implies
 \E[\boldsymbol{x}] \cdot w = \E[y] \\
 \E[(\boldsymbol{x} \cdot w)^2] = E[y^2] \implies 
 w^T \boldsymbol{x}^T  \boldsymbol{x} w = y^T  y
 \end{split}
 \end{equation*}

The latter is a quadratic equation in $w=[w_1, w_2]$. Substitution for $w_1$ using the former gives us the following equation for $w_2$:

$$ a w_2^2 + b w_2 + c = 0$$

where:

$$a = \left(\frac{E[\boldsymbol{x}_2]}{E[\boldsymbol{x}_1]} \right)^2 E[\boldsymbol{x}_1^2] - 2 \left(\frac{E[\boldsymbol{x}_2]}{E[\boldsymbol{x}_1]}\right) E[\boldsymbol{x}_1 \boldsymbol{x}_2] + E[\boldsymbol{x}_2]^2$$ 

$$b = E[\boldsymbol{x}_1 \boldsymbol{x}_2]\left(\frac{E[y]}{E[\boldsymbol{x}_{1}]}\right) - 2\frac{E[y]}{E[\boldsymbol{x}_{1}]} \frac{E[\boldsymbol{x_2}]}{E[\boldsymbol{x}_{1}]}E[\boldsymbol{x}_1^2]$$

$$c =  \left(\frac{E[y]}{E[\boldsymbol{x}_{1}]} \right)^2 E[ \boldsymbol{x}_1^2] -  E[y^2]$$

Here, we use $\boldsymbol{x}_1$ and $\boldsymbol{x}_2$ to refer to the first and second columns of $\boldsymbol{x}$. Using the quadratic formula, we can solve for $w_2$ (getting, in general, two solutions), and then, we can solve for $w_1$:

$$ w_1  = \frac{E[y]}{E[\boldsymbol{x}_{1}]}  - w_2 \frac{E[\boldsymbol{x}_{2}]}{E[\boldsymbol{x}_{1}]} $$

\section{Procedure for Regimes of Efficient Operation}
\label{appendix:regimes}

In this section, we provide more details on how we computed and displayed the regimes of efficient operation for the SM and P1 estimators. In Figure \ref{fig:regimes1}, we conducted 5 simulations for each value of $n$ and $d$, for each of the two estimators, for a total of 600 simulations. In Figure \ref{fig:figure2-2}, we conducted 10 simulations for each value of $R$ and $d$ for a total of 480 simulations. The values of the $d$-dimensional weight vector, $w_0$, were chosen independently and identically from a  Gaussian random variable with mean 0 and variance 1. In each simulation, we chose the SNR level to be 15 dB. The SNR was defined as the average power of the signal, $\boldsymbol{x} \cdot w_0$, divided by the average power of the noise vector, $e$.

We then displayed the estimator that had the best mean relative error across the 5 trials. In some cases, the difference between the SM and P1 estimators was within 2 percentage points; in these cases, the estimator with smaller running time was selected. For example, this meant that the P1 estimator, which reduces to the SM estimator $d=1$ after some computational overhead, was never the optimal estimator in the 1-dimensional setting. We also indicated the range of errors by placing a black dot in each cell corresponding to the mean of the inference error for that choice of parameters. The smallest dot corresponded to an inference error of $<5\%$, the medium-sized dot corresponded to an inference error of between $5$ and $30 \%$, and the largest dot corresponded to an error of more than $30 \%$.

We did not include the LS estimator in these figures because, in general, it produced very similar results to the P1 estimator. Thus,  when the mean or standard deviation were compared, there was negligible difference between the two. But when the \textit{worst-case} performance across the 10 trials of the LS estimator was compared to that of the P1 estimator, the P1 generally produced better results. See Figure \ref{fig:alternate-regime1}.

The P2 estimator was also tested for some values of $n$ and $d$, and provided similar accuracies to the P1 and LS estimator; however, it required significantly more running time, and so was not included in the full analysis. 

\newgeometry{left=4cm,right=4cm,top=0.1cm,bottom=0.1cm}

\begin{figure}
\begin{centering}
\includegraphics[width=0.75\columnwidth]{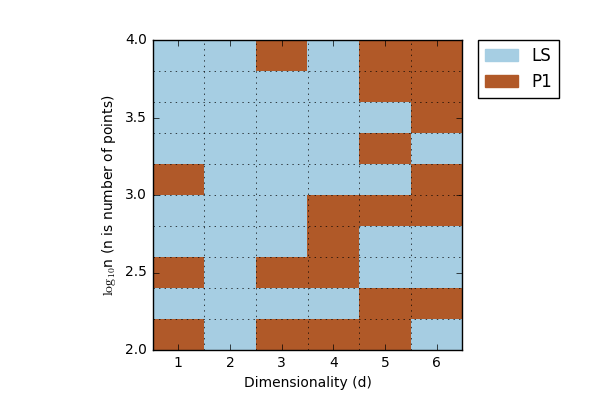}

\medskip
\includegraphics[width=0.75\columnwidth]{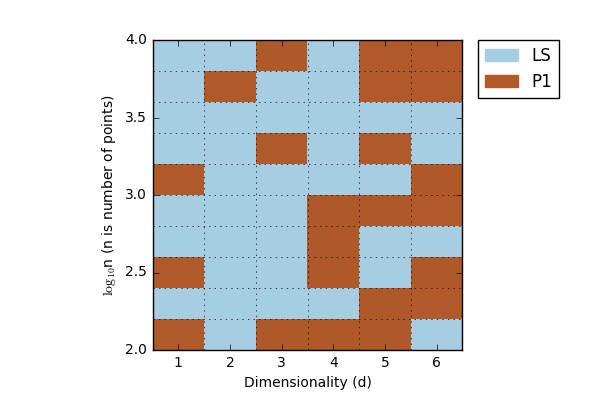}

\medskip
\includegraphics[width=0.75\columnwidth]{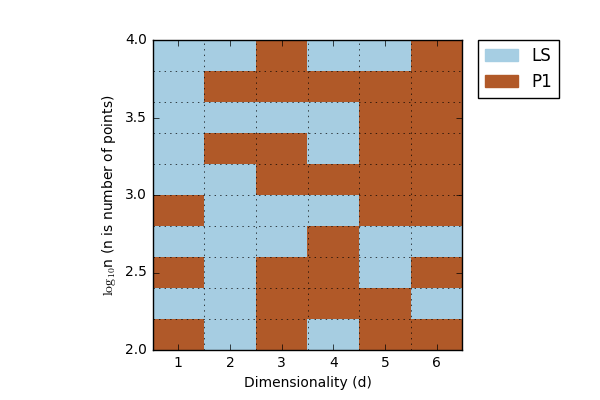}

\caption{Here, we show the results of using various metrics to compare the LS and P1 estimators for each $(n,d)$. There are no clear trends in the mean (top) and mean + one standard deviation (center), particularly for $d>2$ (as the SM estimator performed best for $d \le 2$). When the worst-case relative error across 10 trials was compared (bottom), the P1 estimator did tend to perform better.} \label{fig:alternate-regime1}
\end{centering}
\end{figure}

\restoregeometry

\section{Order-Invariant Estimators}
\label{appendix:order-invariant}

The LS and SM estimators discussed in the main text can be considered special cases of a general class of estimators that we refer to as \textit{order-invariant} estimators. These are estimators that take the following form:

\begin{equation*}
\begin{split}
&L(\boldsymbol{x},y,w) =  \lvert g((\boldsymbol{x} \cdot w)^{\uparrow}) - g(y^{\uparrow}) \rvert_p \\
&\hat{w} = \arg \min_{w} \, L(\boldsymbol{x},y,w)
\end{split}
\end{equation*}

The LS estimator is the specific case when $g(v) = v$, and $p=2$. The SM estimator has $g(v) =  \sum_k^K f(k) \frac{1}{n} \sum_i^n {v_i^{\,k}} $ and $p=2$.

One way to think about these estimators is that they minimize a distance between the \textit{sorted} vectors $(\boldsymbol{x} \cdot w)^{\uparrow}$ and $(y)^{\uparrow}$. Alternatively, they can be thought of as minimizing a distance metric between the two unsorted vectors, but where that distance metric is invariant to the ordering of each of the vector. This resembles the setting of histogram similarity \cite{cha2002measuring}. There are number of more complex metrics for histogram similiarty, such as the earth mover's distance (EMD) \cite{dobrushin1970definition} and the Kolmogorov-Smirnov (KS) statistic \cite{wilcox2005kolmogorov}. These can be converted to estimators by choosing $g(v)$ to be the empirical cumulative probability distribution (ECDF) of the samples in a vector $v$. Thus, we may define:

$$\hat{w}_{\text{EMD}} = \arg \min_{w} \, \lvert \text{ECDF}((\boldsymbol{x} \cdot w)^{\uparrow}) - \text{ECDF}(y^{\uparrow}) \rvert_1$$

$$\hat{w}_{\text{KS}} = \arg \min_{w} \, \lvert \text{ECDF}((\boldsymbol{x} \cdot w)^{\uparrow}) - \text{ECDF}(y^{\uparrow}) \rvert_{\infty}$$

How do these histogram distance-based estimators compare to the LS and SM estimators? In general, we found that they provide no significant advantage over the the SM and LS estimators despite increased computational complexity.

For illustration, see Figure \ref{fig:figureC}, where we plot the relative errors of the EMD and KS estimators and compare them to the other estimators. We have also plotted a particularly poor order-invariant estimator, the \textit{small-D} estimator, which takes the smallest $D$ elements of a vector and computes the squared-difference between them. While these other order-invariant estimators that we have plotted have higher inference errors than the LS and SM estimators, it remains an open question whether there exist other, more optimal order-invariant estimators.

\begin{figure}[]
\centering
\includegraphics[width=0.6\columnwidth]{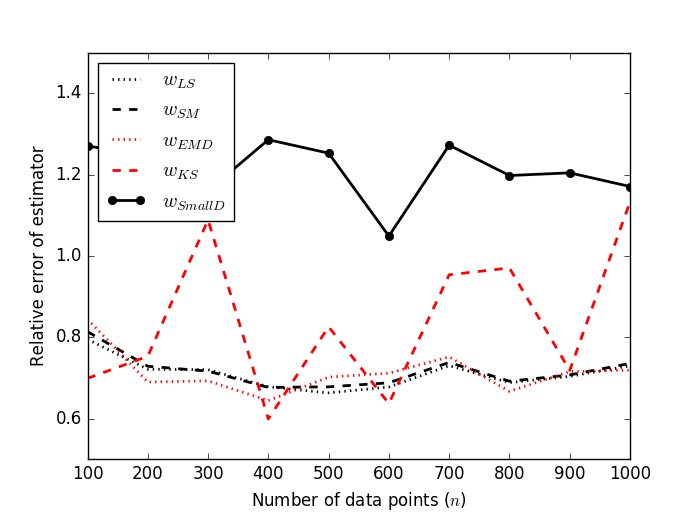}
\caption{The relative errors (averaged over 10 trials) of 5 order-invariant estimators are plotted, as a function of $n$. Here, $d=2$ and $\sigma_E^2=1$. The KS and EMD estimators don't offer increased accuracy despite increased computational complexity. The Small-D estimator offers particularly poor results. Similar trends are observed for other values of $d$ and $n$.}
\label{fig:figureC}
\end{figure}

\section{An expression for the noise-adjusted $w_{SM}$ when $d>2$}
\label{appendix:noise}

We would like to derive the general expression for a consistent, noise-adjusted self-moments (NA-SM) estimator. To do this, we equate the expected value of every sample moment of $(\boldsymbol{x} \cdot \hat{w} + e)$ with the corresponding sample moment of $y$. The $k^{th}$ sample moment of $y$ takes a straightforward form:

$$ \frac{1}{n} \sum_{i}^n y_i^k .$$

The $k^{th}$ moment of $(\boldsymbol{x} \cdot \hat{w} + e)$ is

$$ \frac{1}{n} \sum_{i}^n (\boldsymbol{x} \cdot \hat{w} + e)^k. $$

The expression inside the sum can be expanded in terms of the noise moments using the binomial theorem:

$$  \sum_j^k \binom{k}{j} (x_i \hat{w})^j \E[e^{k-j}] 
. $$

Thus, we can write each of the moment constraints to be:

\begin{equation*}
\label{}
\underbrace{\frac{1}{n}\sum_j^k \sum_{i}^n \binom{k}{j} ( x_i \hat{w})^j \E[e^{k-j}] 
}_{M_k} = \underbrace{\frac{1}{n}\sum_{i}^n y_i^k}_{N_k} 
\end{equation*}

It can be shown that this guarantees the consistency of noise-adjusted SM estimator, if there is a unique solution for $\hat{w}$ (see, for example, Theorem 2.1 in \cite{newey1994large}).

\section{Results from a Negative Control}
\label{appendix:negative-control}

\begin{table}[h]

  \caption{This table represents the results from applying the shuffled regression  estimator to the data. For comparison, a negative control is also included. Each cell represents the relative error, as defined earlier, between the OLS weights and the estimated weights. The mean (over 10 trials) is reported. Errors of less than 0.3 are shaded in green, between 0.3 and 1.0 in yellow, and above 1.0 in red. } \label{table-nc}
  \vspace{3mm}
  \centering

 \begin{tabu} to 0.9\columnwidth { | p{2cm} | X[c] | X[c] | X[c] | X[c] | X[c]|}
    

\cline{2-6}
\multicolumn{1}{c}{} & \multicolumn{5}{|c|}{\textbf{Number of Replications} ($C$) } \\

    \hline
    \textbf{Dataset} & 1 & 2 & 4 & 6 & 8\\\hline 
    
lsat & \cellcolor{red!15} 1.07 & \cellcolor{yellow!25} 0.9 & \cellcolor{yellow!25} 0.84 & \cellcolor{yellow!25} 0.68 & \cellcolor{yellow!25} 0.66 \\ \hline
accidents & \cellcolor{red!15} 1.07 & \cellcolor{red!15} 1.04 & \cellcolor{yellow!25} 0.92 & \cellcolor{yellow!25} 0.92 & \cellcolor{yellow!25} 0.74 \\ \hline
acetylene & \cellcolor{red!15} 1.82 & \cellcolor{red!15} 2.16 & \cellcolor{red!15} 1.43 & \cellcolor{red!15} 2.22 & \cellcolor{red!15} 1.12 \\ \hline
power-plant & \cellcolor{yellow!25} 0.8 & \cellcolor{yellow!25} 0.8 & \cellcolor{yellow!25} 0.8 & \cellcolor{yellow!25} 0.79 & \cellcolor{yellow!25} 0.8 \\ \hline
airfoil & \cellcolor{yellow!25} 0.76 & \cellcolor{yellow!25} 0.75 & \cellcolor{yellow!25} 0.74 & \cellcolor{yellow!25} 0.75 & \cellcolor{yellow!25} 0.75 \\ \hline
yacht & \cellcolor{red!15} 1.99 & \cellcolor{red!15} 1.3 & \cellcolor{red!15} 1.54 & \cellcolor{red!15} 1.82 & \cellcolor{red!15} 1.43 \\ \hline
concrete & \cellcolor{red!15} 1.13 & \cellcolor{red!15} 1.06 & \cellcolor{red!15} 1.04 & \cellcolor{red!15} 1.1 & \cellcolor{red!15} 1.14 \\ \hline
synthetic1 & \cellcolor{yellow!25} 0.92 & \cellcolor{yellow!25} 0.95 & \cellcolor{yellow!25} 0.89 & \cellcolor{yellow!25} 0.9 & \cellcolor{yellow!25} 0.88 \\ \hline
synthetic2 & \cellcolor{yellow!25} 0.97 & \cellcolor{yellow!25} 0.91 & \cellcolor{yellow!25} 0.93 & \cellcolor{yellow!25} 0.93 & \cellcolor{yellow!25} 0.89 \\ \hline
synthetic3 & \cellcolor{yellow!25} 0.9 & \cellcolor{yellow!25} 0.89 & \cellcolor{yellow!25} 0.91 & \cellcolor{yellow!25} 0.88 & \cellcolor{yellow!25} 0.88 \\ \hline
synthetic4 & \cellcolor{red!15} 1.28 & \cellcolor{red!15} 1.25 & \cellcolor{red!15} 1.27 & \cellcolor{red!15} 1.27 & \cellcolor{red!15} 1.25 \\ \hline
synthetic5 & \cellcolor{red!15} 1.01 & \cellcolor{yellow!25} 0.97 & \cellcolor{yellow!25} 0.98 & \cellcolor{yellow!25} 0.96 & \cellcolor{yellow!25} 0.95 \\ \hline
synthetic6 & \cellcolor{yellow!25} 0.82 & \cellcolor{yellow!25} 0.81 & \cellcolor{yellow!25} 0.82 & \cellcolor{yellow!25} 0.79 & \cellcolor{yellow!25} 0.81 \\ \hline
    
\end{tabu}

\end{table}

In Section \ref{subsection:standard}, we examined the accuracy of shuffled linear regression on several real-world datasets. During the course of the experiment, we randomly partitioned the datasets into separate replications. The labels were only shuffled \textit{within} each replication, and not across replications. As such, this raises the concern that by increasing the number of replications, we were effectively disambiguating the permutation matrix (i.e. the ordering of the labels).

To determine whether that was the case, we conducted a baseline experiment -- standard linear regression with a random permutation of the labels within each replication. If we were dividing the data into too many replications, then this baseline would be able to recover the linear weights of the model, just like shuffled linear regression. 

The results of the baseline experiment are found in Table \ref{table-nc}. In each cell is the relative error between the weights determined by a linear regression on random permutation of the labels and the weights determined by linear regression on the correctly-ordered labels (averaged across 10 trials). Because the relative error remains high -- generally above 0.7 -- we conclude that we were not disambiguating the labels to any significant extent. 

\end{appendices}

\end{document}